\documentclass{article}
\PassOptionsToPackage{numbers,sort&compress}{natbib}

\usepackage[preprint]{neurips_2026}
\usepackage{graphicx}
\usepackage[utf8]{inputenc} % allow utf-8 input
\usepackage[T1]{fontenc}    % use 8-bit T1 fonts
\usepackage{hyperref}       % hyperlinks
\usepackage{url}            % simple URL typesetting
\usepackage{booktabs}       % professional-quality tables
\usepackage{amsfonts}       % blackboard math symbols
\usepackage{nicefrac}       % compact symbols for 1/2, etc.
\usepackage{microtype}      % microtypography
\usepackage{xcolor}         % colors
\usepackage{amsmath}
\usepackage[most]{tcolorbox}
\usepackage[dvipsnames]{xcolor}
\usepackage{multirow}
\usepackage{wrapfig}
\usepackage{subcaption}
\usepackage[numbers,sort&compress]{natbib}
\usepackage{amssymb}
\usepackage{amsmath}
\usepackage{amsthm}
\newtheorem{proposition}{Proposition}
\newtheorem{assumption}{Assumption}

\usepackage{enumitem}

\setlist[itemize]{noitemsep,leftmargin=*,topsep=0pt}
\setlist[enumerate]{noitemsep,leftmargin=*,topsep=0pt}

\title{Do Self-Evolving Agents Forget? Capability Degradation and Preservation in Lifelong LLM Agent Adaptation}

\author{
\textbf{Ye Yu\textsuperscript{1}}, 
\textbf{Xiaopeng Yuan\textsuperscript{1}}, 
\textbf{Haibo Jin\textsuperscript{1}}, 
\textbf{Heming Liu\textsuperscript{1}}, 
\textbf{Yaoning Yu\textsuperscript{1}}, 
\textbf{Haohan Wang\textsuperscript{1}\thanks{Corresponding author: haohanw@illinois.edu}}
\\
\textsuperscript{1} University of Illinois Urbana-Champaign, IL, USA
}

\begin{document}

\maketitle

\begin{abstract}
Recent advances in LLM agents enable systems that autonomously refine workflows, accumulate reusable skills, self-train their underlying models, and maintain persistent memory. However, we show that such self-evolution is often non-monotonic: adapting to new task distributions can progressively degrade previously acquired capabilities across all major evolution channels.

We identify this phenomenon as \emph{capability erosion under self-evolution} and show that it consistently emerges across workflow, skill, model, and memory evolution. To mitigate this issue, we propose \emph{Capability-Preserving Evolution} (CPE), a general stabilization principle that constrains destructive capability drift during continual adaptation. Across all four evolution dimensions, CPE consistently improves retained capability stability while preserving adaptation performance. For example, in workflow evolution, CPE improves retained simple-task performance from 41.8\% to 52.8\% under GPT-5.1 optimization while simultaneously achieving stronger complex-task adaptation.

Our findings suggest that stable long-horizon self-evolving agents require not only acquiring new capabilities, but also explicitly preserving previously learned ones during continual adaptation.
\end{abstract}

\section{Introduction}
\label{sec:intro}
LLM agents are rapidly shifting from static, manually engineered systems toward self-evolving entities that continuously refine themselves through interaction with their environments~\citep{gao2026surveyselfevolvingagentswhat}.
Recent frameworks enable agents to autonomously optimize their reasoning workflows~\cite{hu2025automateddesignagenticsystems, zhang2025aflowautomatingagenticworkflow, wang2025evoagentxautomatedframeworkevolving, zhang2025evoflowevolvingdiverseagentic, liu2026sewselfevolvingagenticworkflows}, construct reusable skills and tools~\citep{nguyen2025dynasaurlargelanguageagents, zhang2026memskilllearningevolvingmemory, wang2023voyageropenendedembodiedagent, zhao2024expelllmagentsexperiential, acikgoz2026toolr0selfevolvingllmagents, chen2026skillcraftllmagentslearn}, update their underlying model parameters~\citep{zelikman2022starbootstrappingreasoningreasoning, zeng2025bstarmonitoringbalancingexploration, tian2024selfimprovementllmsimaginationsearching}, and accumulate persistent memory~\citep{suzgun2025dynamiccheatsheettesttimelearning,zeng2025bstarmonitoringbalancingexploration, wei2025evomemorybenchmarkingllmagent} without human intervention.
Together, these advances suggest a compelling long-term vision: agents that can continually expand their competence over their lifetime through continual self-directed adaptation.
 
Yet this long-term vision rests on an assumption that has received little systematic scrutiny: when an agent adapts to new tasks within its task scope, it retains competence on tasks it has already mastered.
Our empirical study indicates that this assumption does not hold in many cases.
Across the major channels through which self-evolving agents adapt--workflow, skill and tools, model, memory~\citep{gao2026surveyselfevolvingagentswhat}--acquiring new competencies frequently comes at the expense of previously mastered ones.
As agents' task distribution temporarily shifts from an old task distribution X toward a new task distribution Y, repeated self-updates gradually weaken the internal structures that previously supported successful behavior on earlier tasks, causing prior capabilities to decay as illustrated in Figure~\ref{fig:method_overview}. 
We term this phenomenon \emph{capability erosion under self-evolution}.

\begin{figure*}[htbp]
    \centering
    \includegraphics[width=\linewidth]{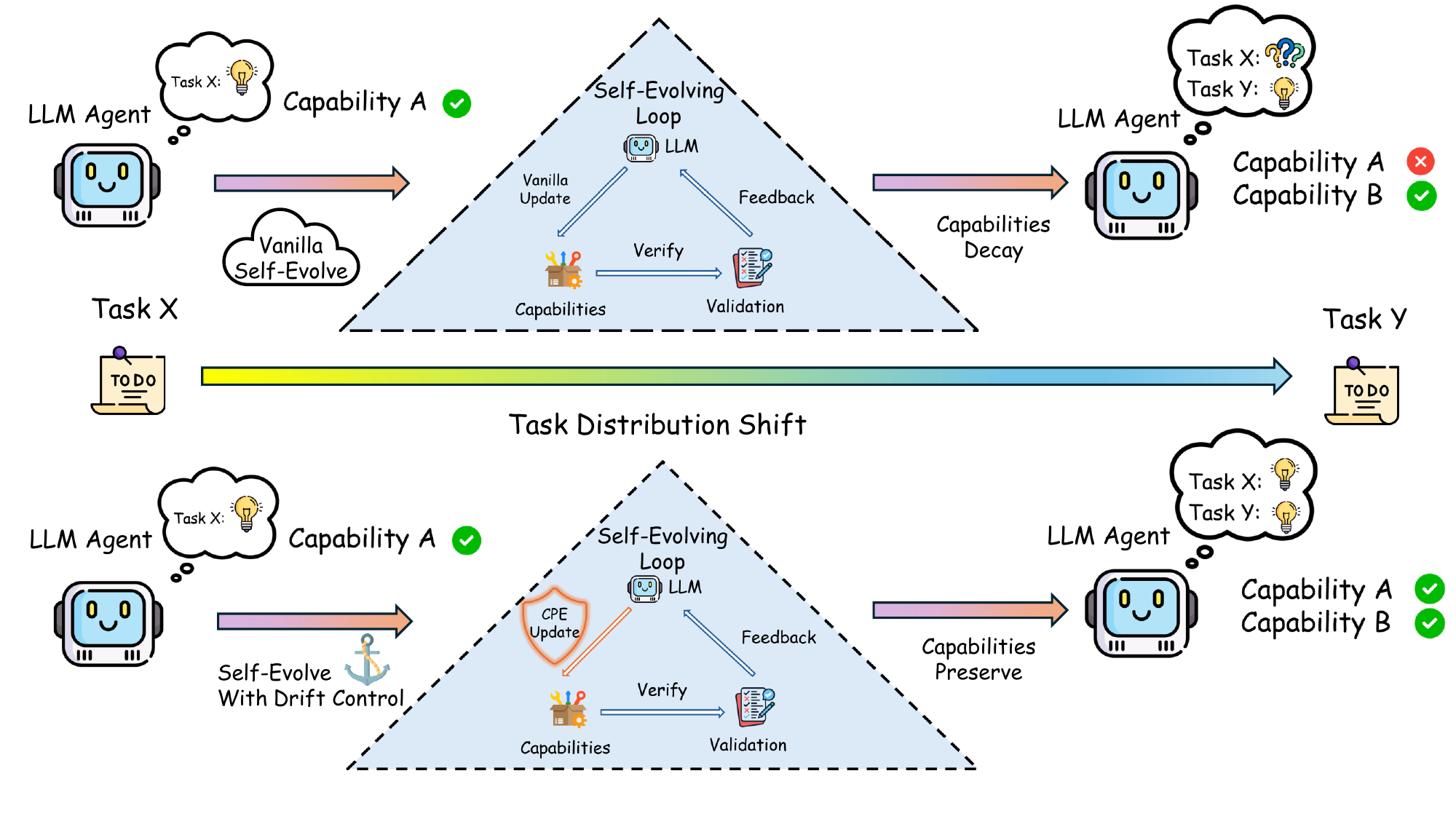}
\vspace{-25pt}
\caption{Self-evolution under task distribution shift. Under unconstrained self-evolution, repeated adaptation toward new tasks progressively erodes capability structures that supported previously mastered tasks, causing prior capability loss. Capability-Preserving Evolution instead constrains self-updates to acquire new capabilities while preserving previously retained ones.}
\label{fig:method_overview}
\vspace{-20pt}
\end{figure*}

Capability erosion cannot be characterized as a single drop in held-out accuracy. It manifests structurally across three dimensions:
\begin{enumerate}
    \item Retrospective capability decay. Performance on previously mastered tasks degrades after fine-tuning on new task distributions.
    \item Behavioral policy drift. Workflows become longer and more complex, introducing unnecessary execution steps that introduce noise.
    \item Generalization erosion. Adaptation to the evolved distribution narrows the effective reasoning space and interferes with previously acquired capabilities during skill or memory evolution.
\end{enumerate}
Taken together, these modes reveal that unconstrained self-evolution does not accumulate capability monotonically, but instead introduces continual drift away from previously reliable behavior. Adaptation to new tasks often trades off with the retention of prior capabilities.
 
A single mechanistic account underlies this phenomenon across all four evolution dimensions. 
Although workflow, tool, model, and memory evolution appear to be heterogeneous, all four dimensions involve repeatedly rewriting a mutable repository to accommodate new tasks. Each update optimized solely for a new task distribution introduces \emph{destructive interference}: modifications to internal repository structures that previously supported successful behavior on earlier tasks are degraded or overwritten.
Capability erosion is therefore not an artifact of any particular adaptation mechanism, but a structural consequence of unconstrained continual self-modification~\citep{Kirkpatrick_2017, wang2024comprehensivesurveycontinuallearning, marcus2025swebenchclcontinuallearning, fountas2025suresurprisedrivenprioritised, chen2024analyzingreducingcatastrophicforgetting, zhang2026mechanisticanalysiscatastrophicforgetting, luo2024empiricalstudycatastrophicforgetting}.

This diagnosis points toward a shared mitigation principle: self-evolution updates should acquire new competencies while minimizing destructive disruption to capability structures already encoded in the repository. In other words, the objective is not merely to move the agent toward the new task distribution Y, but to do so while preserving the capability support required by prior tasks, as depicted in the lower pipeline of Figure~\ref{fig:method_overview}. We term this \emph{Capability-Preserving Evolution} and propose it as a general design principle for stable long-horizon agent adaptation. To this end, we analyze how capability erosion manifests across all four evolution dimensions and, from this analysis, derive dimension-specific regularization strategies that instantiate capability-preserving evolution in each setting. Across workflow, skill and tool, model, and memory evolution, these strategies substantially mitigate observed erosion, demonstrating that capability-preserving evolution constitutes a broadly applicable stabilization mechanism for self-evolving agents.
 
Our contributions are threefold:
\begin{itemize}

\item We identify capability erosion as a fundamental failure mode of self-evolving LLM agents, where continual adaptation degrades previously acquired capabilities across workflow, skill, model, and memory evolution.

\item We provide a unified interference perspective on capability erosion, together with a local analysis showing how unconstrained self-evolution overwrites prior capability-supporting structures.

\item We propose Capability-Preserving Evolution (CPE), a general stabilization principle for self-evolving agents that constrains destructive capability drift during continual adaptation while preserving adaptation performance.

\end{itemize}

\section{Capability Erosion under Distributional Self-Evolution}
\label{sec:problem_formulation}
\subsection{Sequential Self-Evolution}

We model self-evolution as a sequential adaptation process over task distributions.
Let \(D_1,D_2,\ldots,D_T\) denote the distributions encountered by an agent over its lifetime.
At stage \(t\), the agent is represented by a capability state \(R_t \in \mathcal{R}\), which may correspond to an executable workflow, a skill repository, trainable model parameters, or a memory store. Self-evolution updates \(R_t\) to improve performance on newly encountered task distributions while implicitly affecting performance on previously learned ones.

\textbf{Notation.}
A sampled task instance is denoted by \(x \sim D_t\).
The instantaneous loss on \(x\) under capability state \(R\) is written as \(\ell(R;x)\).
The expected stage loss is
\begin{equation}
    \mathcal{L}_t(R)
=
\mathbb{E}_{x\sim D_t}[\ell(R;x)].
\end{equation}
Given a new distribution \(D_t\), naïve self-evolution updates the capability state by optimizing only the current-stage objective:
\begin{equation}
    R_t^{\mathrm{naive}}
\in
\arg\min_{R\in\mathcal{R}}
\mathcal{L}_t(R).
\end{equation}
This abstraction intentionally hides the concrete representation of \(R_t\).
For workflow evolution, \(R_t\) is an executable policy graph or prompt program.
For skill/tool evolution, \(R_t\) is a bounded repository of reusable procedural entries.
For model evolution, \(R_t\) is the trainable parameter state.
For memory evolution, \(R_t\) is an external retrieved memory store.
Although the underlying representations differ, all four settings share the same learning dynamic: optimizing performance on the current distribution \(D_t\) can interfere with capabilities previously acquired on earlier distributions \(D_{<t}\).

\subsection{Capability Erosion}

Let \(D_{<t}=\{D_1,\ldots,D_{t-1}\}\) denote previously encountered distributions.
We define the retained old-distribution risk at stage \(t\) as
\begin{equation}
    \mathcal{L}_{<t}(R)
=
\sum_{i<t}\alpha_i \mathcal{L}_i(R),
\qquad
\alpha_i\ge 0,\quad \sum_{i<t}\alpha_i=1.
\end{equation}
The coefficients \(\alpha_i\) control the relative importance assigned to prior distributions and may be chosen uniformly in the simplest setting.

We say capability erosion occurs at stage \(t\) if adaptation toward the current distribution \(D_t\) increases retained risk on previously encountered distributions \(D_{<t}\):
\begin{equation}
    \mathcal{L}_{<t}(R_t) > \mathcal{L}_{<t}(R_{t-1}).
\end{equation}
The central source of erosion is distributional interference. 
If the minimizers of the current objective \(\mathcal{L}_t\) are misaligned with those of prior objectives \(\mathcal{L}_{<t}\), then optimizing only \(\mathcal{L}_t\) can move the capability state toward regions that increase retained old-distribution risk. 
When old and new distributions rely on overlapping capability support, such interference may be weak or even beneficial. Conversely, larger mismatch between the structures required by \(D_t\) and \(D_{<t}\) induces stronger capability erosion.

Importantly, this formulation is stated in terms of a general capability state \(R_t\) rather than only neural parameters. In self-evolving agents, adaptation may occur through workflow rewriting, skill replacement, memory updates, or parameter optimization, all of which modify a mutable capability repository. Despite these different implementations, the same interference principle applies: adaptation toward new distributions can disrupt structures that previously supported successful behavior.
\subsection{A Local Mechanism of Capability Erosion}

% The discussion above characterizes capability erosion at the level of task distributions and retained risk. To better understand how such erosion emerges during self-evolution updates, we now consider a local approximation in which the capability state admits a continuous surrogate representation. This analysis is not
% intended to fully characterize arbitrary agent repositories. Rather, it isolates the common
% interference geometry that arises whenever self-evolution admits a continuous
% local surrogate for the mutable capability state.

The discussion above characterizes capability erosion at the level of task distributions and retained risk. To study how such erosion emerges during self-evolution updates, we now consider a local approximation in which small changes to the capability state induce corresponding local changes in the retained and current-task objectives.

Let $\Delta = R-R_{t-1}$ denote a local update around the current capability
state. Suppose $R_{t-1}$ is locally well-adapted to the retained old
distributions, so that
\[
    \nabla L_{<t}(R_{t-1})=0,
\]
and let
\[
    H_{<t}=\nabla^2 L_{<t}(R_{t-1}) \succeq 0,
    \qquad
    g_t=\nabla L_t(R_{t-1}).
\]
Here, $H_{<t}$ captures directions that are locally important for retained
capabilities, while $g_t$ is the update direction preferred by the new
distribution.

\begin{proposition}[Curvature mechanism of capability erosion]
\label{prop:local-erosion}
Let $R_{t-1}$ be a local minimizer of $L_{<t}$, and suppose $L_{<t}$ is twice
differentiable around $R_{t-1}$ with Hessian $H_{<t}\succeq 0$. Consider a
naive current-task update
\[
    R_t^{\mathrm{naive}} = R_{t-1}-\eta g_t,
    \qquad
    g_t=\nabla L_t(R_{t-1}),
\]
where $\eta>0$. Then, for sufficiently small $\eta$,
\[
    L_{<t}(R_t^{\mathrm{naive}})-L_{<t}(R_{t-1})
    =
    \frac{\eta^2}{2}g_t^\top H_{<t}g_t
    +
    o(\eta^2).
\]
Consequently, naive adaptation increases retained old-distribution loss
whenever $g_t$ has nonzero projection onto a positive-curvature direction of
$H_{<t}$.
\end{proposition}

Proposition~\ref{prop:local-erosion} shows that capability erosion is not
controlled only by the size of the update. It is governed by whether the
new-task update direction overlaps with directions that are locally important
for retained capabilities. Thus, updates that improve the current distribution
can be destructive when they move the capability state along old-task-sensitive
directions. The proof is given in Appendix~\ref{app:local-proofs}.

\section{Capability-Preserving Evolution}
\label{sec:method}
\subsection{Regularized Self-Evolution Objective}

Capability-Preserving Evolution addresses this interference by adding an explicit retention regularizer.
At stage \(t\), instead of optimizing only \(\mathcal{L}_t\), CPE solves
\begin{equation}
    R_t^{\mathrm{CPE}}
\in
\arg\min_{R\in\mathcal{R}}
\mathcal{L}_t(R)
+
\lambda \Omega_t(R,R_{t-1}),
\label{eq:cpe_objective}
\end{equation}
where \(\Omega_t\) measures deviation from prior capability-supporting structure and \(\lambda>0\) controls the stability--plasticity tradeoff. Naïve self-evolution corresponds to the special case \(\lambda=0\), where adaptation is optimized without any explicit retention constraint.

The role of \(\Omega_t\) is not to prevent evolution.
Rather, it biases self-evolution toward low-interference solutions: among updates that improve the new distribution, prefer those that minimally disrupt structures needed by prior distributions.
Different evolution dimensions instantiate \(\Omega_t\) differently:
\begin{equation}
\Omega_t(R,R_{t-1})=
\begin{cases}
\Omega_{\mathrm{wf}} & \text{workflow retention regularization},\\
\Omega_{\mathrm{sk}} & \text{skill retention regularization},\\
\Omega_{\mathrm{md}} & \text{parameter retention regularization},\\
\Omega_{\mathrm{mem}} & \text{memory retention regularization}.
\end{cases}
\end{equation}

\subsection{Local Forgetting Control under CPE}

We now show how capability-preserving evolution controls the local erosion
mechanism above. In a local quadratic surrogate, CPE chooses an update
$\Delta_\lambda$ by solving
\[
    \Delta_\lambda
    =
    \arg\min_{\Delta}
    \left\{
        g_t^\top \Delta
        +
        \frac{1}{2}\Delta^\top H_t\Delta
        +
        \frac{\lambda}{2}\Delta^\top M_t\Delta
    \right\},
\]
where $H_t\succeq 0$ is the local curvature of the current-task loss and
$M_t\succ 0$ is a preservation metric. The matrix $M_t$ is a local quadratic
surrogate for the preservation regularizer $\Omega_t(R,R_{t-1})$.

\begin{assumption}[Capability-aligned preservation metric]
\label{assump:cap-aligned}
There exists a constant $c>0$ such that, locally around $R_{t-1}$,
\[
    H_{<t}\preceq cM_t .
\]
Equivalently, for every local update direction $\Delta$,
\[
    \Delta^\top H_{<t}\Delta
    \le
    c\Delta^\top M_t\Delta .
\]
\end{assumption}

Assumption~\ref{assump:cap-aligned} does not require the preservation metric to
equal the retained-task Hessian. It only requires the metric to upper-control
directions along which retained-task loss is locally sensitive. In model
evolution, this corresponds to Fisher-style parameter-importance penalties; in
workflow, skill, and memory evolution, it corresponds to a capability-aware
discrepancy over the mutable repository.

\begin{proposition}[CPE suppresses local old-task forgetting]
\label{prop:cpe-local-control}
Suppose $R_{t-1}$ is a local minimizer of $L_{<t}$, $L_{<t}$ is twice
differentiable around $R_{t-1}$ with Hessian $H_{<t}\succeq 0$, and
Assumption~\ref{assump:cap-aligned} holds. Let $H_t\succeq 0$, $M_t\succ 0$,
and $\lambda>0$. Then
\[
    \Delta_\lambda
    =
    -(H_t+\lambda M_t)^{-1}g_t,
\]
and the retained old-distribution degradation satisfies
\[
    L_{<t}(R_{t-1}+\Delta_\lambda)-L_{<t}(R_{t-1})
    \le
    \frac{c}{2\lambda^2}
    \|g_t\|_{M_t^{-1}}^2
    +
    o(\|\Delta_\lambda\|^2).
\]
Moreover, $\|\Delta_\lambda\|=O(\lambda^{-1})$ for large $\lambda$, so the
remainder is lower order than the leading $O(\lambda^{-2})$ term.
\end{proposition}

Proposition~\ref{prop:cpe-local-control} gives a local stability interpretation
of CPE. Naive adaptation can increase old-task loss by moving along
old-task-sensitive curvature directions. A capability-aligned preservation
metric penalizes precisely these directions, so increasing $\lambda$ suppresses
old-task degradation while also limiting the size of the new-task update. This
is the local stability--plasticity tradeoff that motivates the
dimension-specific CPE regularizers below. The proof is given in
Appendix~\ref{app:local-proofs}.

\subsection{Channel-Specific Instantiations}
The theory above gives a shared template.
Each empirical setting instantiates the same objective. 
% with a different preservation regularizer.

\textbf{Workflow evolution.}
Here \(R_t\) is an executable workflow. The regularizer \(\Omega_{\mathrm{wf}}\) discourages deviations from previously successful execution behavior during adaptation to new task distributions.

\textbf{Skill/tool evolution.}
Here \(R_t\) is a bounded skill repository. The regularizer \(\Omega_{\mathrm{sk}}\) discourages unnecessary removal of previously useful skills during sequential adaptation.

\textbf{Model evolution.}
Here \(R_t=\theta_t\) denotes trainable model parameters or adapter weights.
CPE uses an importance-weighted quadratic penalty:
\begin{equation}
    \Omega_{\mathrm{md}}(\theta,\theta_{t-1})
=
\frac{1}{2}\sum_i F_i(\theta_i-\theta_{t-1,i})^2,
\end{equation}
where \(F_i\) estimates the importance of parameter \(i\) to prior-stage performance.
This regularizer discourages parameter movement along directions that are important for previously learned distributions, reducing interference between old and new task adaptation.

\textbf{Memory evolution.}
Here \(R_t\) is a persistent memory bank. The regularizer \(\Omega_{\mathrm{mem}}\) discourages deletion or suppression of memories useful for previously encountered distributions.
\section{Empirical Instantiations of Capability-Preserving Evolution}
\label{sec:experiements}
We now instantiate Capability-Preserving Evolution across four major self-evolution dimensions:
workflow evolution, skill/tool evolution, model evolution, and memory evolution.
Although the underlying repositories differ across these settings, all four exhibit the same
stability–plasticity tension: adaptation toward new task distributions can interfere with
previously retained capabilities.
For each dimension, we implement a dimension-specific preservation mechanism while
keeping the underlying self-evolution pipeline unchanged, allowing direct comparison
between unconstrained and capability-preserving evolution.

\subsection{CPE for Workflow Evolution}

\begin{wraptable}{r}{0.38\linewidth}
\vspace{-8pt}
\centering
\footnotesize
\resizebox{1\linewidth}{!}{%
\begin{tabular}{l|cc|cc}
\toprule
\multirow{2}{*}{\textbf{Env.}}
& \multicolumn{2}{c|}{\textbf{GPT-5.1}}
& \multicolumn{2}{c}{\textbf{GPT-5 nano}} \\
\cmidrule(lr){2-3} \cmidrule(lr){4-5}
& \textbf{Van.} & \textbf{CPE}
& \textbf{Van.} & \textbf{CPE} \\
\midrule

\multicolumn{5}{c}{\textit{Simple}} \\
\midrule
Airline & 57.1 & 64.3 & 57.1 & 71.4 \\
Retail  & 43.3 & 56.7 & 50.0 & 50.0 \\
Telecom & 25.0 & 37.5 & 33.3 & 45.8 \\
\midrule
Avg.    & 41.8 & 52.8 & 46.8 & 55.7 \\

\midrule
\multicolumn{5}{c}{\textit{Complex}} \\
\midrule
Airline & 37.5 & 37.5 & 20.8 & 29.2 \\
Retail  & 20.3 & 45.3 & 40.6 & 43.8 \\
Telecom & 14.0 & 17.5 & 7.0 & 7.0 \\
\midrule
Avg.    & 23.9 & 33.4 & 22.8 & 26.7 \\

\bottomrule
\end{tabular}}

\vspace{-4pt}
\caption{
Workflow self-evolution comparison between vanilla self-evolve and CPE-style evolution (in \%). Scores are reported after self-evolving from a shared simple-task seed workflow on the complex-task subset.
}
\label{tab:workflow_backbone_compare}
\vspace{-20pt}
\end{wraptable}

\begin{figure*}[t]
    \centering
    
    \begin{subfigure}[t]{0.32\linewidth}
        \centering
        \includegraphics[width=\linewidth]{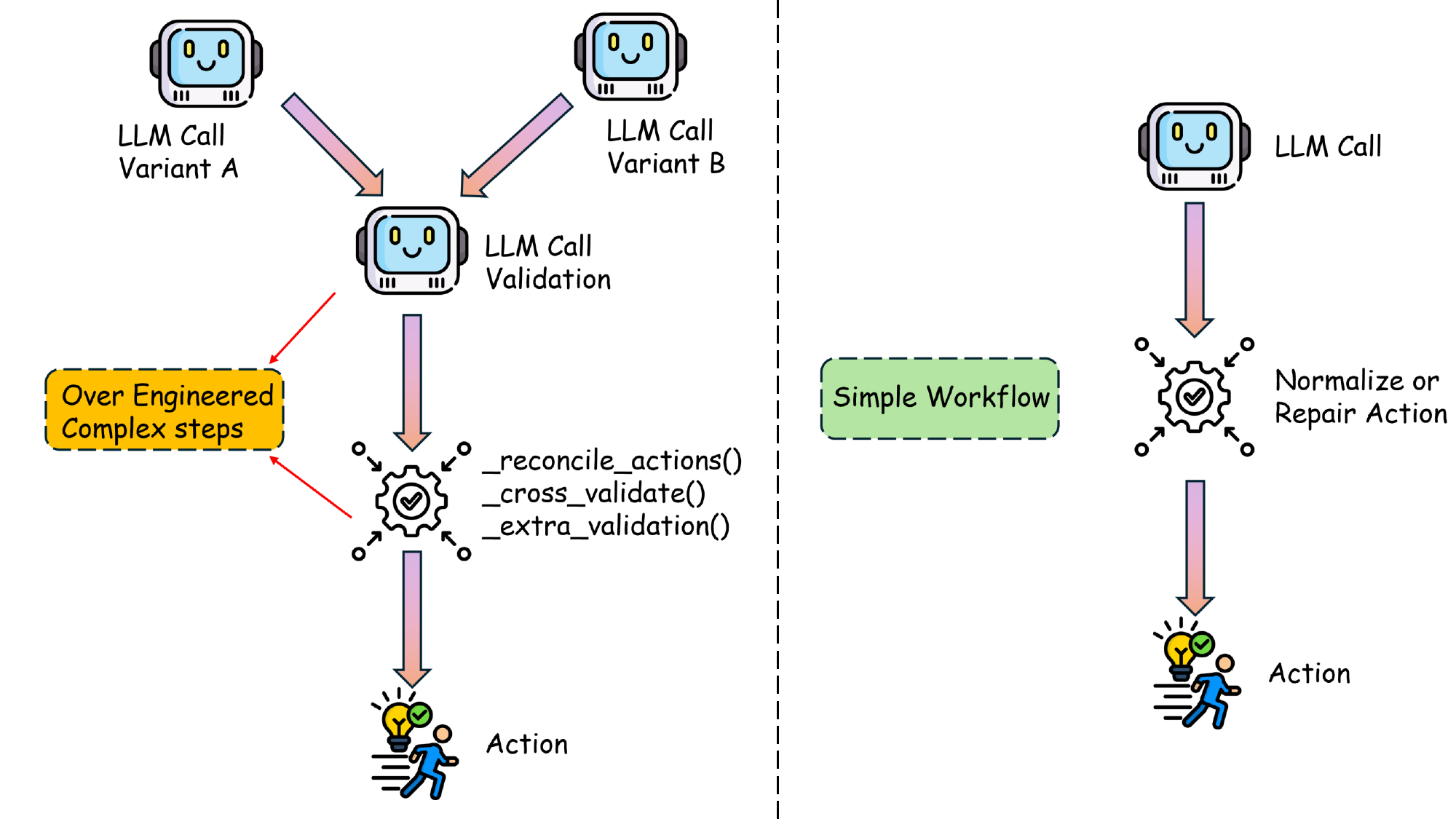}
        \caption{Final workflow comparison after complex-task self-evolution. 
Vanilla EvoAgentX introduces overly complex workflow, while CPE preserves a more compact execution structure.}
        \label{fig:workflow_comparison}
    \end{subfigure}
    \hfill
    \begin{subfigure}[t]{0.32\linewidth}
        \centering
        \includegraphics[width=\linewidth]{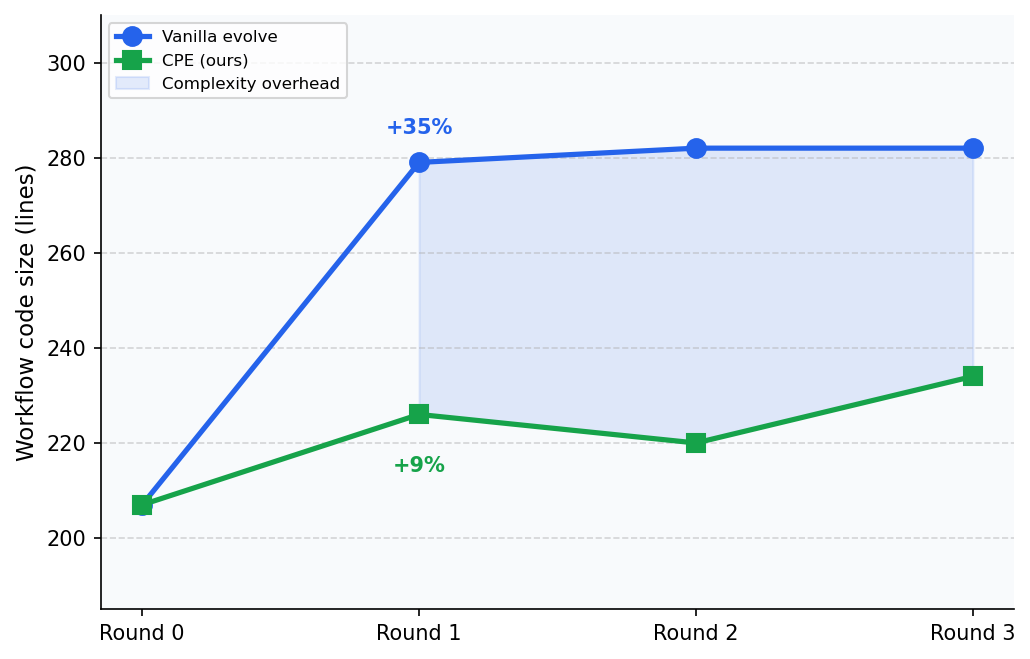}
        \caption{Workflow complexity measured during self-evolution on complex Airline task subset. Vanilla EvoAgentX  accumulates workflow complexity, while CPE substantially slows long-horizon structural growth.}
        \label{fig:workflow_complexity}
    \end{subfigure}
    \hfill
    \begin{subfigure}[t]{0.32\linewidth}
        \centering
        \includegraphics[width=\linewidth]{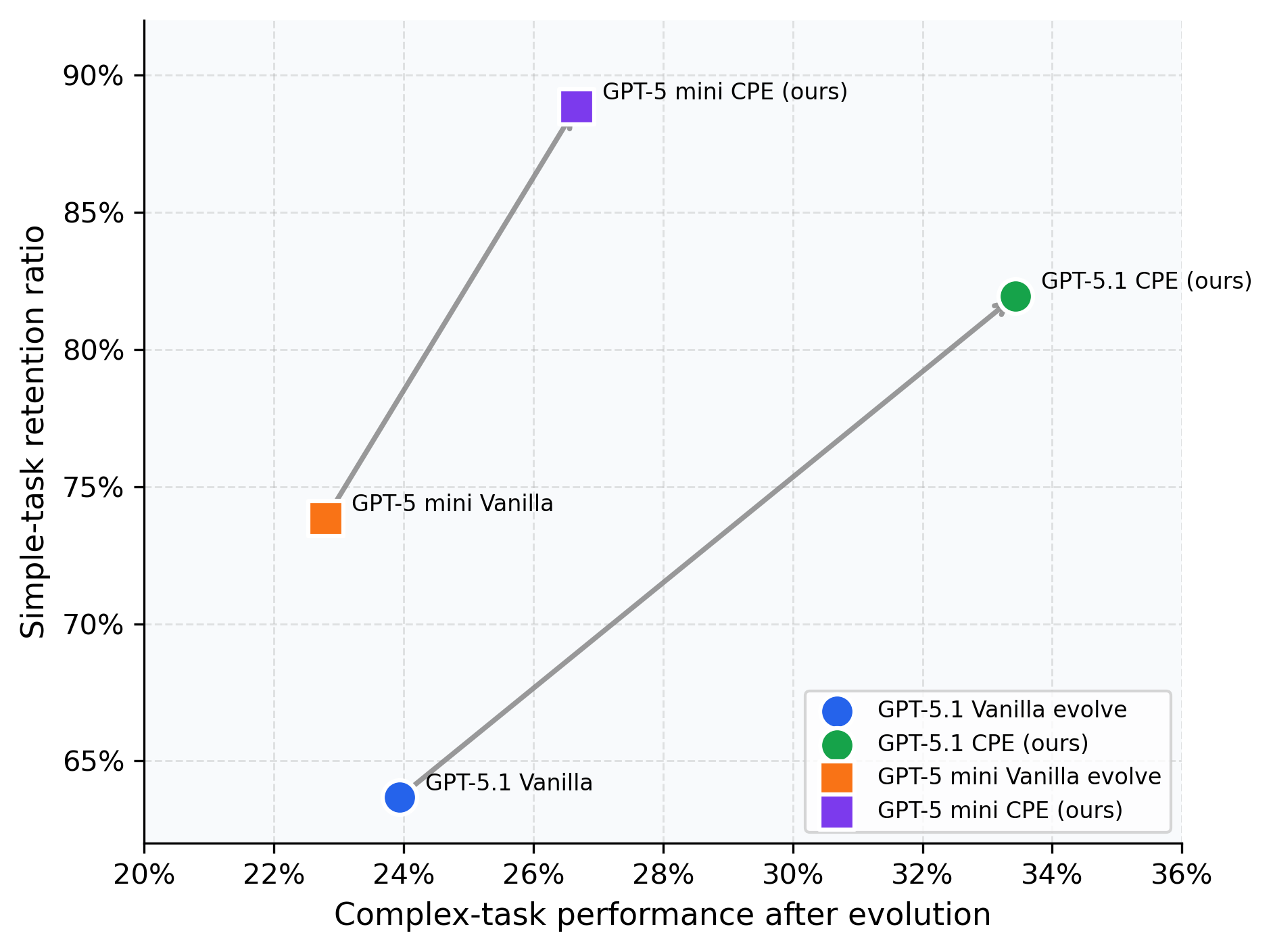}
        \caption{Post-evolution trade-off between retained simple-task performance and complex-task performance. Each point denotes the final evolved workflow after evolution on complex task subset. }
        \label{fig:workflow_retention}
    \end{subfigure}
    
\caption{Behavioral policy drift during complex-task workflow self-evolution. 
Starting from the same simple-task seed workflow, vanilla EvoAgentX and CPE undergo parallel self-evolution on the complex-task subset. 
Vanilla evolution accumulates structural detours, producing a more complex final executable workflow. }
\label{fig:workflow_combined}
\end{figure*}

% In workflow self-evolution, the mutable repository
% $R_t^{\mathrm{wf}}$ is the current executable workflow,
% including both operator orchestration and prompt specifications.
% Naïve workflow mutation repeatedly rewrites this artifact
% to improve validation performance on new tasks.
% Over long-horizon evolution, however, such unconstrained rewrites
% often introduce prompt drift and structural policy drift,
% causing later workflows to deviate from the compact execution
% semantics encoded in earlier successful workflows.

% To mitigate this issue, CPE extracts anchor behavioral signatures
% from the seed workflow, including core task objectives,
% output constraints, and critical failure-avoidance behaviors.
% These anchors are injected into each optimization round as
% stability references that discourage disruptive
% workflow rewrites while still permitting adaptation to harder tasks.

\textbf{Setup.}
In workflow self-evolution, the mutable repository
$R_t^{\mathrm{wf}}$ is the current executable workflow,
including both operator orchestration and prompt specifications.
Repeated unconstrained workflow rewriting often introduces
prompt and structural policy drift, causing later workflows
to deviate from the compact execution semantics encoded in
earlier successful behaviors.

We study this setting using EvoAgentX~\citep{wang2025evoagentxautomatedframeworkevolving}
on $\tau^2$-Bench~\citep{barres2025tau2benchevaluatingconversationalagents}
with GPT-5.1 and GPT-5 nano~\citep{singh2026openaigpt5card, openai2024gpt4technicalreport} as optimizer backbones.
For each domain, we first construct a shared seed workflow
on the simple-task subset, then launch parallel vanilla and
CPE self-evolution branches on the complex-task subset.
In both settings, failed trajectories are iteratively fed back
to an optimizer LLM that rewrites the workflow to improve
future execution success.
Vanilla EvoAgentX performs unconstrained workflow mutation,
whereas CPE injects anchor behavioral signatures extracted
from the seed workflow, including core task objectives,
output constraints, and failure-avoidance behaviors,
to discourage disruptive rewrites while preserving the same
optimizer, feedback loop, and mutation budget.

After evolution, we evaluate the final workflows on both
simple and complex tasks to measure retained routine-task
robustness and adaptation to harder task distributions.
Additional implementation details are deferred to
Appendix~\ref{app:workflow_impl}.

\textbf{Observation and Analysis.}
Table~\ref{tab:workflow_backbone_compare} and Figure~\ref{fig:workflow_combined} show that, starting from the same seed workflow and optimization budget, vanilla EvoAgentX and CPE produce markedly different post-evolution policies. Although both adapt to complex tasks, vanilla EvoAgentX yields a more bloated workflow and larger degradation on previously mastered simple tasks.

As illustrated in Figure~\ref{fig:workflow_comparison}, vanilla evolution often inserts extra validation operators and structural detours. This pattern is locally reasonable: in long-horizon complex tasks, moderate validation can prevent malformed intermediate outputs from derailing execution. However, when such validation-heavy repairs are repeatedly accepted without compactness constraints, the workflow gradually becomes over-defensive. Slightly nonstandard but semantically valid actions may be rejected, triggering conservative fallback behavior or prematurely terminating recoverable episodes. This structural drift is reflected in Figure~\ref{fig:workflow_complexity}, where vanilla EvoAgentX accumulates much longer workflow code, while CPE substantially slows this expansion. Figure~\ref{fig:workflow_retention} further shows that CPE achieves a better retention--adaptation trade-off, preserving more simple-task performance while maintaining comparable or stronger complex-task gains. 

Overall, we find that unconstrained self-evolution repeatedly rewards locally beneficial repairs that improve difficult trajectories but gradually overwrite the compact execution semantics needed for stable routine-task behavior. CPE mitigates this failure by constraining evolution toward seed-consistent executable behavior.

\subsection{CPE for Skill/Tool Evolution}
\begin{figure*}[t]
    \centering
    \includegraphics[width=\linewidth]{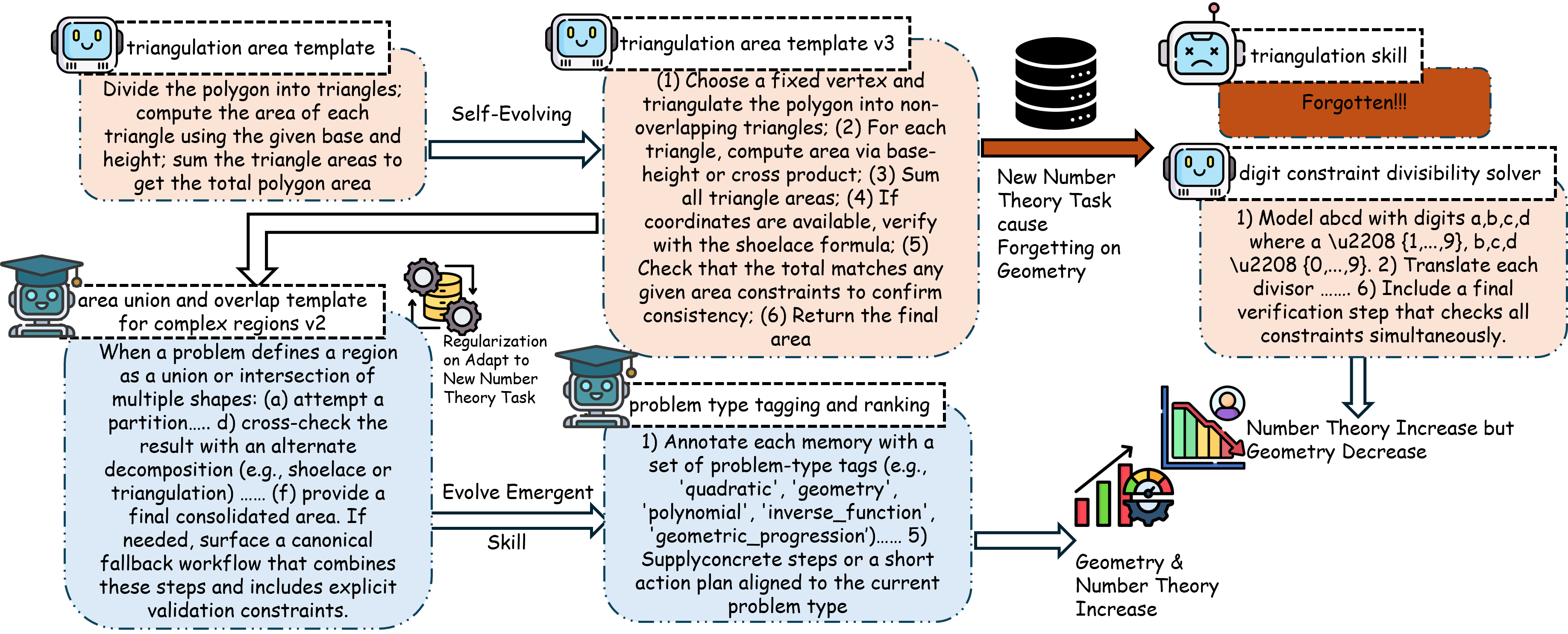}
    \vspace{-10pt}
\caption{Generalization erosion during skill self-evolution. As the task distribution shifts from Geometry to Number Theory, vanilla self-evolution progressively overwrites previously useful Geometry skills under bounded repository capacity. In contrast, CPE preserves prior procedural knowledge through capability-preserving skill consolidation, allowing old and new skills to coexist more stably.}
\label{fig:skill_evolve}

\vspace{-15pt}
\end{figure*}

\begin{figure*}[t]
    \centering
    
    \begin{subfigure}[t]{0.48\linewidth}
        \centering
        \includegraphics[width=\linewidth]{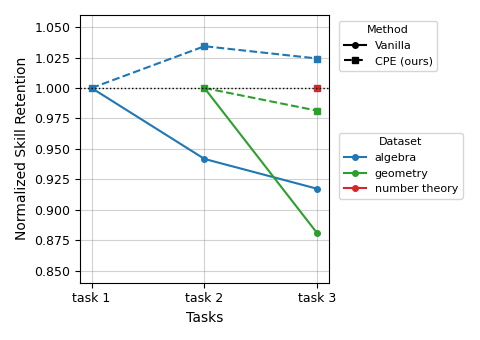}
        \caption{Normalized retained skill usage across sequential skill evolution stages. Vanilla self-evolution gradually suppresses previously learned  skills after the task distribution shifts, while CPE preserves broader prior-skill utilization through capability-preserving repository management.}
        \label{fig:skill_usage}
    \end{subfigure}
    \hfill
    \begin{subfigure}[t]{0.48\linewidth}
        \centering
        \includegraphics[width=\linewidth]{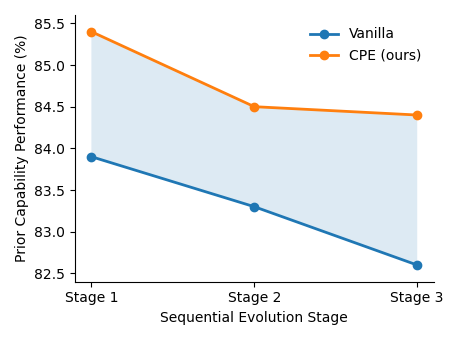}
        \caption{Sequential self-evolution performance on the Algebra evaluation set using GPT-4o mini. As the evolution trajectory shifts toward later domains, vanilla self-evolution exhibits continual degradation on earlier Algebra capability, whereas CPE substantially stabilizes retained performance throughout the evolution.}
        \label{fig:skill_evolve_performance_curve}
    \end{subfigure}
    
     \caption{Generalization erosion during skill self-evolution. Vanilla self-evolution and CPE undergo identical evolution trajectories. Vanilla evolution progressively overwrites previously useful procedural skills as the task distribution shifts, while CPE preserves historically reusable skills through capability-preserving repository compression and merge-based capacity management.}
    \label{fig:skill_combined}
\vspace{-10pt}
\end{figure*}

% In skill self-evolution, the mutable repository is a bounded-capacity skill bank that stores reusable procedural reasoning templates.
% As new hard cases emerge, the agent continually synthesizes new specialized skills to improve downstream performance.

% Under naïve self-evolution, however, the limited repository capacity creates a continual overwrite dynamic.
% As new specialized skills are added, previously useful general-purpose skills become less frequently retrieved and more likely to be evicted, causing the repository to drift toward recently optimized subdomains.

% CPE mitigates this issue through capability-preserving repository consolidation.
% Instead of directly overwriting old entries, semantically related skills are merged to release capacity, while historically reusable high-utility skills are preferentially preserved.\hbc{seems redundant here, since below will mention it again, just delete it?}

\textbf{Setup.}
We study skill self-evolution using a  MemSkill~\citep{zhang2026memskilllearningevolvingmemory}
style framework on MATH~\citep{hendrycks2021measuringmathematicalproblemsolving} with GPT-4o mini
and GPT-5 nano~\citep{openai2024gpt4technicalreport, singh2026openaigpt5card}
as the designer and executor backbones.
The mutable repository is a bounded-capacity skill bank that stores reusable procedural reasoning templates.
During each episode, the agent retrieves the most relevant skills from the current repository to guide downstream reasoning, while failed trajectories are accumulated and periodically analyzed by a designer LLM to synthesize new procedural skills.

Under vanilla self-evolution, accepted skills are directly added into the shared repository, and least-frequently-used entries are evicted once the capacity budget is exceeded.
As the task distribution shifts, newly synthesized specialized skills progressively dominate repository capacity, causing previously useful general-purpose skills to become under-utilized and eventually overwritten.

Our CPE variant preserves the same retrieval, execution, and skill-generation pipeline, but regularizes repository updates before admission.
Candidate skills are checked against preserved capability signatures, semantically related skills are merged to release capacity, and historically reusable high-utility skills are protected from direct deletion.
Both variants are evaluated under identical evolution budgets.
Implementation details are deferred to Appendix~\ref{app:tool_impl}.

\begin{wraptable}{r}{0.48\textwidth}
\vspace{-10pt}
\centering
\footnotesize
\setlength{\tabcolsep}{4pt}
\resizebox{1\linewidth}{!}{%
\begin{tabular}{l|cc|cc}
\toprule
\multirow{2}{*}{\textbf{Domain}}
& \multicolumn{2}{c|}{\textbf{GPT-5 nano}}
& \multicolumn{2}{c}{\textbf{GPT-4o mini}} \\
\cmidrule(lr){2-3} \cmidrule(lr){4-5}
& \textbf{Vanilla} & \textbf{CPE}
& \textbf{Vanilla} & \textbf{CPE} \\
\midrule

Algebra       & 84.3 & 85.9 & 82.6 & 84.4 \\
Geometry      & 58.9 & 59.3 & 50.8 & 51.1 \\
Number Theory & 85.4 & 87.8 & 73.5 & 76.1 \\

\bottomrule
\end{tabular}}
\vspace{-4pt}
\caption{Skill dimension self-evolution comparison between vanilla self-evolve and CPE-style evolution (in \%). Scores are reported after Algebra $\rightarrow$ Geometry $\rightarrow$ Number Theory evolution.}
\label{tab:tool_selfevolve}
\vspace{-12pt}
\end{wraptable}

\textbf{Observation and Analysis.}
Table~\ref{tab:tool_selfevolve}, Figure~\ref{fig:skill_evolve},  and Figure~\ref{fig:skill_combined} show that capability erosion in skill self-evolution mainly arises from repository overwrite under sequential task-distribution shift. As the task distribution shifts from Geometry to Number Theory, newly synthesized skills gradually dominate retrieval frequency and bounded repository capacity, causing previously useful Geometry skills to become under-utilized and eventually evicted.

CPE mitigates this issue through capability-preserving consolidation. Instead of directly overwriting old entries, semantically related skills are merged before new skills are admitted, allowing earlier capabilities to remain represented within the repository. This effect appears in both repository statistics and downstream performance. Figure~\ref{fig:skill_usage} shows that vanilla self-evolution progressively suppresses retained skill usage for earlier domains, while CPE preserves substantially higher prior-skill utilization across stages. Similarly, Figure~\ref{fig:skill_evolve_performance_curve} shows continual degradation on the Algebra evaluation set under vanilla evolution, whereas CPE maintains more stable retained capability throughout sequential evolution.

Overall, vanilla skill evolution progressively overwrites previously useful skills under bounded repository capacity, while CPE mitigates this erosion through capability-preserving consolidation.

\subsection{CPE for Model Evolution}
\begin{wrapfigure}{r}{0.4\linewidth}
    \vspace{-25pt}
    \centering
    \includegraphics[width=\linewidth]{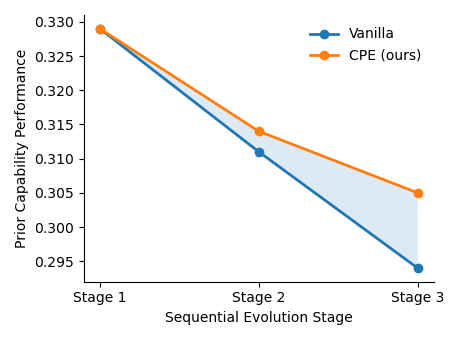}
    \vspace{-15pt}
    \caption{Sequential self-evolution performance on the Anatomy evaluation set using Qwen3-0.6B. 
%     Vanilla
% self-evolution exhibits continual degradation on earlier
% anatomy capability, whereas CPE substantially stabilizes
% retained performance throughout the evolution process.
}
    \label{fig:model_evolve_curve}
    \vspace{-15pt}
\end{wrapfigure}

\textbf{Setup.}
We study model self-evolution using STaR self-training~\citep{zelikman2022starbootstrappingreasoningreasoning}
on MedMCQA~\citep{pal2022medmcqalargescalemultisubject}
with Qwen3-0.6B~\citep{yang2025qwen3technicalreport}
and Llama3.2-3B~\citep{grattafiori2024llama3herdmodels}.
The mutable repository is the trainable parameter space of the evolving language model adapter.
At each stage, the model generates rationale-answer traces on a newly introduced medical domain,
filters correct trajectories, and performs LoRA fine-tuning on the resulting demonstrations,
producing a sequential self-training process across domains.

Under vanilla self-evolution, later-stage updates continually adapt the model toward newly encountered domains,
but often overwrite parameter regions that previously supported earlier reasoning behaviors,
producing classical catastrophic forgetting.

Our CPE variant preserves the same continual self-training pipeline,
but adds an Elastic Weight Consolidation style regularizer~\citep{Kirkpatrick_2017,zheng2026revisitingweightregularizationlowrank}.
After each stage, Fisher-based parameter importance weights are estimated to discourage destructive updates on parameters important to prior domains while still permitting adaptation on lower-importance directions. Both variants are evaluated under identical evolution order, optimization budget, and LoRA configurations.
Implementation details are deferred to Appendix~\ref{app:model_impl}.

\begin{wraptable}{r}{0.4\linewidth}
% \vspace{-12pt}
\centering
\footnotesize
\setlength{\tabcolsep}{4pt}
\resizebox{1\linewidth}{!}{%
\begin{tabular}{l|cc|cc}
\toprule
\multirow{2}{*}{\textbf{Domain}}
& \multicolumn{2}{c|}{\textbf{Qwen3-0.6B}}
& \multicolumn{2}{c}{\textbf{Llama3.2-3B}} \\
\cmidrule(lr){2-3} \cmidrule(lr){4-5}
& \textbf{Vanilla} & \textbf{CPE}
& \textbf{Vanilla} & \textbf{CPE} \\
\midrule
Anatomy      & 29.4 & 30.5 & 68.9 & 70.3 \\
Biochemistry & 37.2 & 38.4 & 79.4 & 79.5 \\
Dental       & 32.2 & 33.0 & 53.0 & 53.5 \\
\bottomrule
\end{tabular}}
\vspace{-5pt}
\caption{Model dimension self-evolution comparison between vanilla self-evolve and CPE-style evolution (in \%). Scores are reported after the full Anatomy $\rightarrow$ Biochemistry $\rightarrow$ Dental evolution.}
\vspace{-10pt}
\label{tab:cl_medical}
\end{wraptable}

\textbf{Observation and Analysis.} Table~\ref{tab:cl_medical} and Figure~\ref{fig:model_evolve_curve} show that capability erosion in model self-evolution emerges as continual parametric overwrite under sequential self-training. Vanilla self-evolution exhibits continual degradation on earlier domains as later-stage adaptation proceeds. This effect is consistent across both backbone scales in Table~\ref{tab:cl_medical}. On the other hand, CPE consistently preserves higher retained performance than vanilla self-evolution across all evaluated domains. On Qwen3-0.6B, retained accuracy improves from 29.4\% to 30.5\% on Anatomy, from 37.2\% to 38.4\% on Biochemistry, and from 32.2\% to 33.0\% on Dental, with similar trends on Llama3.2-3B.

CPE mitigates this erosion by constraining destructive parameter movement through Fisher-based importance regularization, encouraging adaptation along lower-interference directions while preserving previously useful representations.

\subsection{CPE for Memory Evolution}
% \hwc{please summarize the memory results here. like one or two paragraphs}

Memory evolution also exhibits capability erosion: vanilla memory updates can evict, suppress, or compete with previously useful memories, degrading old-task retention. CPE mitigates this interference through evidence-gated preservation, protecting historically reliable memories while keeping low-evidence entries adaptable, reducing the average retention gap from 2.3\% to 0.7\%. Due to space constraints, we defer the full setup and detailed analysis to Appendix~\ref{app:memory_impl}.
\section{Related Work}
\textbf{Self-Evolving LLM Agents}. Recent work enables LLM agents to refine themselves across four dimensions: workflow optimization~\cite{hu2025automateddesignagenticsystems, zhang2025aflowautomatingagenticworkflow, wang2025evoagentxautomatedframeworkevolving, zhang2025evoflowevolvingdiverseagentic, liu2026sewselfevolvingagenticworkflows}, skill and tool accumulation~\citep{nguyen2025dynasaurlargelanguageagents, zhang2026memskilllearningevolvingmemory, wang2023voyageropenendedembodiedagent, zhao2024expelllmagentsexperiential, acikgoz2026toolr0selfevolvingllmagents, chen2026skillcraftllmagentslearn}, model self-training~\citep{zelikman2022starbootstrappingreasoningreasoning, zeng2025bstarmonitoringbalancingexploration, tian2024selfimprovementllmsimaginationsearching}, and memory evolution~\citep{suzgun2025dynamiccheatsheettesttimelearning,zeng2025bstarmonitoringbalancingexploration, wei2025evomemorybenchmarkingllmagent}.
\citet{gao2026surveyselfevolvingagentswhat, fang2025comprehensivesurveyselfevolvingai} survey these as complementary paths toward open-ended improvement. Critically, all existing frameworks measure self-evolution solely by forward progress on new tasks, without asking whether adaptation preserves competence on previously mastered ones. Our work is the first to systematically study this question, revealing capability degradation as a consistent, structural failure mode across all four dimensions.

% \textbf{Self-Evolving LLM Agents}
% A growing body of work enables LLM agents to autonomously refine themselves across four primary dimensions: workflow optimization~\cite{hu2025automateddesignagenticsystems, zhang2025aflowautomatingagenticworkflow, wang2025evoagentxautomatedframeworkevolving}, skill and tool accumulation~\citep{nguyen2025dynasaurlargelanguageagents, zhang2026memskilllearningevolvingmemory}, model self-training~\citep{zelikman2022starbootstrappingreasoningreasoning, zeng2025bstarmonitoringbalancingexploration}, and persistent memory evolution ~\citep{suzgun2025dynamiccheatsheettesttimelearning,zeng2025bstarmonitoringbalancingexploration}. ~\citet{gao2026surveyselfevolvingagentswhat} survey this landscape and frame these dimensions as complementary paths toward open-ended agent improvement. Critically, however, all existing frameworks evaluate self-evolution by measuring forward progress on newly encountered tasks. None ask whether adaptation to new tasks preserves competence on previously mastered ones. Our work is the first to systematically study this question, revealing that capability degradation is a consistent and structural failure mode across all four evolution dimensions.

\textbf{Continual Learning}. Catastrophic forgetting is well-studied in continual learning~\citep{wang2024comprehensivesurveycontinuallearning, marcus2025swebenchclcontinuallearning, fountas2025suresurprisedrivenprioritised, chen2024analyzingreducingcatastrophicforgetting, zhang2026mechanisticanalysiscatastrophicforgetting, luo2024empiricalstudycatastrophicforgetting}, where methods such as Elastic Weight Consolidation~\citep{Kirkpatrick_2017} penalize destructive parameter updates. However, this literature assumes the agent's only mutable component is its weights. Self-evolving agents break this assumption by simultaneously modifying workflows, skill repositories, memory stores, and model parameters, each capable of inducing degradation through distinct mechanisms.
Our work extends the forgetting problem to this multi-component setting, identifying dimension-specific erosion modes and targeted preservation strategies.

% \textbf{Continual Learning}.
% The risk of new updates overwriting previously learned behaviors is well-studied in the continual learning literature~\citep{ wang2024comprehensivesurveycontinuallearning}, where regularization-based methods such as Elastic Weight Consolidation~\citep{Kirkpatrick_2017} explicitly penalize destructive parameter movement. However, this line of work addresses forgetting exclusively at the model parameter level, under the assumption that the agent's only mutable component is its weights. Self-evolving agents break this assumption: they simultaneously modify workflows, skill repositories, memory stores, and model parameters, each of which can independently induce capability degradation through distinct interference mechanisms. Our work extends the forgetting problem to this broader multi-component setting, identifying dimension-specific erosion modes and deriving targeted preservation strategies for each.

\textbf{Risks in Self-Evolving Agents}. Recent work identifies safety risks under autonomous self-modification: \citet{shao2026agentmisevolveemergentrisks} document emergent safety misalignment, while \citet{zhou2026capabilityorientedtraininginducedalignment} link capability-oriented training to alignment drift. Our work shares this motivation but focuses on retained competence degradation, specifically whether agents lose previously mastered functional capabilities as they adapt to new task distributions, rather than safety-oriented degeneration.

% \textbf{Risks in Self-Evolving Agents}
% Recent studies have also begun to identify that autonomous self-improvement may produce safety degradation under autonomous self-modification. ~\citet{shao2026agentmisevolveemergentrisks} document emergent safety misalignment risks in self-evolving agents, while ~\citet{zhou2026capabilityorientedtraininginducedalignment} identify capability-oriented training as a source of safety alignment drift. Our work is related in motivation but differs substantially in focus. Rather than studying safety-oriented degeneration, we examine retained competence degradation: whether self-evolving agents lose previously mastered functional capabilities as they adapt toward new task distributions.
\section{Conclusion}
\label{sec:conclusion}
We identify capability erosion as a general failure mode of self-evolving LLM agents: adaptation to new task distributions can degrade previously acquired capabilities across workflow, skill, model, and memory evolution. We further show that this phenomenon arises from a shared interference mechanism underlying continual self-modification. To address this issue, we propose Capability-Preserving Evolution, a general principle for stabilizing long-horizon self-evolution by constraining destructive capability drift during adaptation. Our empirical findings show that CPE consistently improves retained capability stability and plasticity, highlighting the importance of preserving previously acquired capabilities during continual self-evolution.
% Across all four evolution dimensions, CPE consistently improves retained capability stability while preserving adaptation performance. 
% Our findings suggest that future self-evolving agents should optimize not only for acquiring new capabilities, but also for preserving previously learned ones throughout continual adaptation.
\newpage

\bibliographystyle{unsrtnat}
\bibliography{paper}
\clearpage
\appendix

\section{Proofs for the Local CPE Analysis}
\label{app:local-proofs}

\begin{proof}[Proof of Proposition~\ref{prop:local-erosion}]
By Taylor expansion of $L_{<t}$ around $R_{t-1}$, for a local update $\Delta$,
\[
    L_{<t}(R_{t-1}+\Delta)
    =
    L_{<t}(R_{t-1})
    +
    \nabla L_{<t}(R_{t-1})^\top \Delta
    +
    \frac{1}{2}\Delta^\top H_{<t}\Delta
    +
    o(\|\Delta\|^2).
\]
Since $R_{t-1}$ is a local minimizer of $L_{<t}$, we have
$\nabla L_{<t}(R_{t-1})=0$. Substituting
$\Delta=-\eta g_t$ gives
\[
    L_{<t}(R_t^{\mathrm{naive}})
    -
    L_{<t}(R_{t-1})
    =
    \frac{\eta^2}{2}g_t^\top H_{<t}g_t
    +
    o(\eta^2).
\]
If $g_t$ has nonzero projection onto a positive-curvature direction of
$H_{<t}$, then $g_t^\top H_{<t}g_t>0$, so the leading term is positive for
sufficiently small $\eta$. This proves the result.
\end{proof}

\begin{proof}[Proof of Proposition~\ref{prop:cpe-local-control}]
The first-order optimality condition for the local CPE objective is
\[
    g_t + H_t\Delta_\lambda+\lambda M_t\Delta_\lambda=0.
\]
Since $M_t\succ 0$, $H_t\succeq 0$, and $\lambda>0$, the matrix
$H_t+\lambda M_t$ is positive definite. Therefore,
\[
    \Delta_\lambda=-(H_t+\lambda M_t)^{-1}g_t.
\]

Let
\[
    A=M_t^{-1/2}H_tM_t^{-1/2}\succeq 0,
    \qquad
    \widetilde g=M_t^{-1/2}g_t .
\]
Then
\[
    \Delta_\lambda
    =
    -M_t^{-1/2}(A+\lambda I)^{-1}\widetilde g .
\]
Therefore,
\[
    \Delta_\lambda^\top M_t\Delta_\lambda
    =
    \widetilde g^\top (A+\lambda I)^{-2}\widetilde g .
\]
Since $A\succeq 0$, every eigenvalue of $(A+\lambda I)^{-2}$ is at most
$\lambda^{-2}$. Hence
\[
    \Delta_\lambda^\top M_t\Delta_\lambda
    \le
    \frac{1}{\lambda^2}\|\widetilde g\|_2^2
    =
    \frac{1}{\lambda^2}g_t^\top M_t^{-1}g_t
    =
    \frac{1}{\lambda^2}\|g_t\|_{M_t^{-1}}^2 .
\]
By Assumption~\ref{assump:cap-aligned}, $H_{<t}\preceq cM_t$, so
\[
    \frac{1}{2}\Delta_\lambda^\top H_{<t}\Delta_\lambda
    \le
    \frac{c}{2}\Delta_\lambda^\top M_t\Delta_\lambda
    \le
    \frac{c}{2\lambda^2}\|g_t\|_{M_t^{-1}}^2 .
\]
Finally, since $R_{t-1}$ is a local minimizer of $L_{<t}$, Taylor expansion gives
\[
    L_{<t}(R_{t-1}+\Delta_\lambda)-L_{<t}(R_{t-1})
    =
    \frac{1}{2}\Delta_\lambda^\top H_{<t}\Delta_\lambda
    +
    o(\|\Delta_\lambda\|^2).
\]
Combining the previous two displays proves the stated bound.

It remains only to justify that the remainder is lower order for large
$\lambda$. Since
\[
    \Delta_\lambda
    =
    -M_t^{-1/2}(A+\lambda I)^{-1}\widetilde g,
\]
and $\|(A+\lambda I)^{-1}\|_2\le \lambda^{-1}$, we have
$\|\Delta_\lambda\|=O(\lambda^{-1})$ when $M_t$ is fixed and positive definite.
Thus the Taylor remainder $o(\|\Delta_\lambda\|^2)$ is lower order than the
leading $O(\lambda^{-2})$ term.
\end{proof}

\section{CPE for Memory Evolution}
\begin{table}[htbp]
\vspace{-10pt}
\centering
\footnotesize
\setlength{\tabcolsep}{4pt}
\begin{tabular}{l|cc|cc}
\toprule
\multirow{2}{*}{\textbf{Domain}}
& \multicolumn{2}{c|}{\textbf{GPT-5 nano}}
& \multicolumn{2}{c}{\textbf{GPT-4o mini}} \\
\cmidrule(lr){2-3} \cmidrule(lr){4-5}
& \textbf{Van.}
& \textbf{Reg.}
& \textbf{Van.}
& \textbf{Reg.} \\
\midrule
Counting \& Probability & 85.0 & 86.0 & 74.5 & 78.0 \\
Prealgebra              & 82.5 & 86.0 & 88.5 & 88.5 \\
Number Theory           & 85.0 & 86.5 & 74.0 & 74.5 \\
\bottomrule
\end{tabular}
\vspace{2pt}
\caption{Memory self-evolution comparison between vanilla self-evolve and CPE-style evolution. Scores are reported after self-evolving from a shared simple-task memory on the complex-task subset.}
\label{tab:memory_selfevolve_compare}
%\vspace{-20pt}
\end{table}

In memory self-evolution, the mutable repository is a persistent external memory bank that continually accumulates reusable experiences to guide future inference. As the agent encounters new task streams, new memory entries may be added, existing memories may be revised, and low-evidence memories may be removed to improve future performance.

Under naïve memory evolution, however, continual updates can gradually interfere with previously useful memories.
As new task-specific memories accumulate, earlier memories become increasingly vulnerable to eviction, suppression, or retrieval competition, causing the memory bank to drift toward recently optimized task distributions.
Consequently, previously useful retrieval patterns may deteriorate even when the underlying model parameters remain fixed.

CPE mitigates this issue through evidence-gated memory preservation.
Instead of treating all memories as equally mutable, memory entries accumulate stability through repeated self-verified support over time.
Historically reliable memories are protected from destructive rewrite or eviction whenever possible, while lower-evidence entries remain adaptable to new task distributions.
This stabilizes long-horizon memory evolution while still permitting continual memory growth and adaptation.

\paragraph{Setup.}
We study memory self-evolution using a Dynamic-Cheatsheet-style external strategy repository~\citep{suzgun2025dynamiccheatsheettesttimelearning} with GPT-4o mini~\citep{openai2024gpt4technicalreport} and GPT-5 nano~\citep{singh2026openaigpt5card} as the underlying solver backbones.
The base model remains fixed throughout the experiment, and only the external memory bank is updated during self-evolution.

The memory repository stores reusable strategy notes accumulated from previous reasoning trajectories.
During sequential evolution, the agent continually updates this repository using self-verified successful experiences from newly encountered task streams.
The vanilla baseline directly performs unconstrained memory insertion, revision, and eviction during continual updates.

Our CPE variant preserves the same memory update pipeline, but introduces evidence-gated memory preservation.
Memories that accumulate repeated self-verified support are treated as increasingly stable and are protected from destructive rewrite or eviction whenever possible, while lower-evidence entries remain adaptable to new task distributions.

Following an old-task adaptation stage, we continue memory evolution on a new-task stream and evaluate retained performance on held-out old-task probes to measure memory-level interference.
Additional implementation details are deferred to Appendix~\ref{app:memory_impl}.

\begin{figure}[t]
    \centering
    \includegraphics[width=\linewidth]{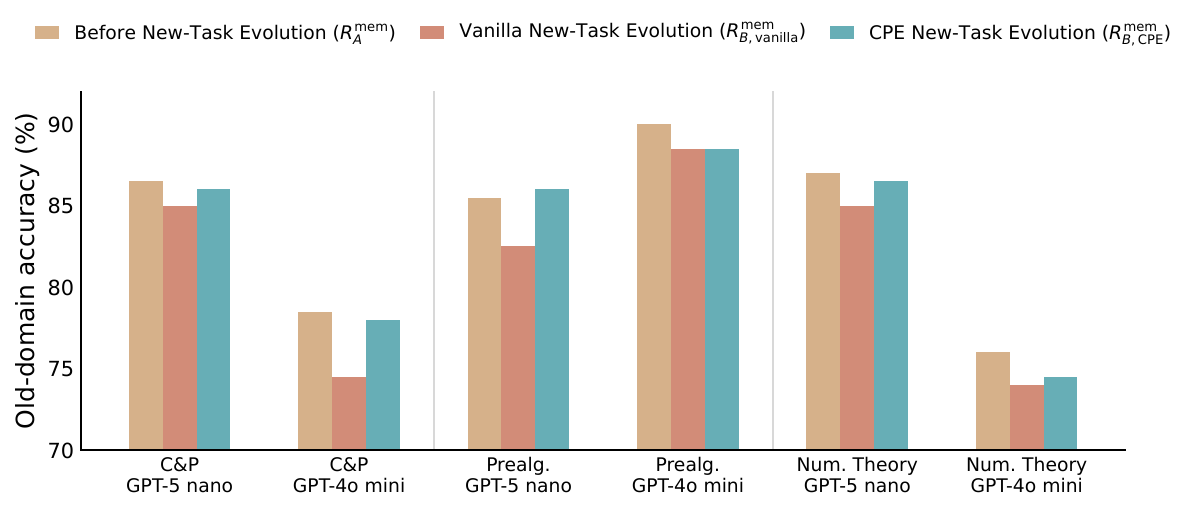}
    \caption{
    \textbf{Memory-level interference under sequential memory evolution.}
    We compare retained old-task performance before and after exposing the memory repository to new-task updates.
    Although the underlying model remains fixed, naïve memory evolution consistently degrades old-task retention after continual memory updates.
    CPE mitigates this interference by stabilizing historically reliable memories during long-horizon memory evolution.
    }
    \label{fig:memory_retention}
\end{figure}

\paragraph{Observation and Analysis.}
Figure~\ref{fig:memory_retention} compares retained old-domain probe accuracy across one pre-evolution memory state and two post-evolution memory states. The pre-evolution state is the memory bank obtained after old-domain adaptation. The first post-evolution state is produced by vanilla new-task memory evolution, where the repository is updated through unconstrained insertion, revision, and eviction. The second post-evolution state is produced by CPE-regularized memory evolution, which uses the same update pipeline but protects historically reliable memories through evidence-gated preservation. This comparison measures how continued memory updates on new tasks affect retained performance on previously learned domains.

Vanilla new-task memory evolution increases old-task interference. Across all six domain--backbone settings, average old-domain performance drops from 83.9\% before new-task evolution to 81.6\% after unconstrained memory updates. This indicates that repository-level memory modification can degrade previously acquired memory benefits even when the underlying model parameters remain fixed.

CPE substantially mitigates this degradation. After CPE-regularized memory evolution, average retained performance reaches 83.3\%, recovering most of the capability loss introduced by naïve updates. The effect is particularly strong on GPT-5 nano, where retained performance after evolution remains nearly identical to the pre-evolution memory baseline. GPT-4o mini shows a smaller but still consistent recovery trend.

These results suggest that persistent memory functions as a mutable capability repository. Under unconstrained evolution, newly introduced memories gradually interfere with previously useful retrieval patterns through eviction and context competition. CPE reduces this interference by preserving stable high-utility memories and preferentially compressing lower-evidence entries under memory-budget constraints.

\section{Workflow Evolve Implementation Details}
\label{app:workflow_impl}

\subsection{Simple and Complex Task Split}
\label{app:task_classification}

\begin{tcolorbox}[colback=white,colframe=blue!40!gray,title=\textbf{Tau2-Bench Task Classification Prompt}]
\textbf{Objective:}  
Classify each Tau2-Bench text task into exactly one label: \texttt{simple}, \texttt{complex}, or \texttt{drop}.

\vspace{0.5em}
\textbf{Rubric:}
\begin{itemize}
    \item \textbf{simple}: single main objective, limited branching, limited state tracking, little planning, and minimal tool use.
    
    \item \textbf{complex}: multiple dependent steps, meaningful planning, multi-stage reasoning, troubleshooting, topic shifts, verification, or substantial tool use/state tracking.
    
    \item \textbf{drop}: ambiguous or in-between cases that should not be forced into either simple or complex.
\end{itemize}

\vspace{0.5em}
\textbf{Output Format:}

Return strict JSON with the following schema:
\begin{verbatim}
{
  "decisions": [
    {
      "task_id": "...",
      "label": "simple|complex|drop",
      "reason": "short reason"
    }
  ]
}
\end{verbatim}

\vspace{0.5em}
\textbf{Input Fields:}
\begin{itemize}
    \item \texttt{Domain}: \texttt{\{domain\}}
    \item \texttt{Tasks}: serialized task summaries in JSON format.
\end{itemize}
\end{tcolorbox}

\begin{table*}[h]
\centering

% ===================== Row 1 =====================

\begin{minipage}[t]{0.45\linewidth}
\centering
\footnotesize
\setlength{\tabcolsep}{5pt}
\begin{tabular}{lccc}
\toprule
\textbf{Domain} & \textbf{Simple} & \textbf{Complex} & \textbf{Dropped} \\
\midrule
Airline & 14 & 29 & 7 \\
Retail  & 30 & 79 & 5 \\
Telecom & 24 & 71 & 19 \\
\bottomrule
\end{tabular}
\captionof{table}{Task split statistics across domains. Tasks labeled as ``Dropped'' were excluded because the classifier judged them ambiguous between simple and complex categories.}
\label{tab:task_split}
\end{minipage}
\hfill
\begin{minipage}[t]{0.45\linewidth}
\centering
\footnotesize
\setlength{\tabcolsep}{3pt}
\begin{tabular}{lccc}
\toprule
\textbf{Domain} & \textbf{Simple Mean} & \textbf{Complex Mean} & \textbf{p-value} \\
\midrule
Airline & 0.86 & 3.93 & 0.0003 \\
Retail  & 3.87 & 5.08 & 0.0119 \\
Telecom & 1.67 & 6.07 & $< 0.001$ \\
\midrule
\textbf{Pooled} & \textbf{2.47} & \textbf{5.28} & \textbf{$< 0.001$} \\
\bottomrule
\end{tabular}
\captionof{table}{Comparison of required tool-call actions (\texttt{num\_actions}) between simple and complex tasks. Complex tasks consistently require more tool calls across all domains.}
\label{tab:num_actions}
\end{minipage}

\vspace{0.8em}

% ===================== Row 2 =====================

\begin{minipage}[t]{0.45\linewidth}
\centering
\footnotesize
\setlength{\tabcolsep}{5pt}
\begin{tabular}{lcc}
\toprule
\textbf{Domain} & \textbf{Simple Mean} & \textbf{Complex Mean} \\
\midrule
Airline & 1.57 & 3.00 \\
Retail  & 0.37 & 0.61 \\
Telecom & 0.00 & 0.00 \\
\bottomrule
\end{tabular}
\captionof{table}{Comparison of the number of natural language assertions (\texttt{num\_nl\_assertions}) between simple and complex tasks. NL assertions are used by $\tau^2$-Bench to measure behavioural constraints that cannot be verified solely through database state.}
\label{tab:nl_assertions}
\end{minipage}
\hfill
\begin{minipage}[t]{0.45\linewidth}
\centering
\footnotesize
\setlength{\tabcolsep}{3pt}
\begin{tabular}{lccc}
\toprule
\textbf{Domain} & \textbf{Simple Mean} & \textbf{Complex Mean} & \textbf{p-value} \\
\midrule
Airline & 22.5 words & 28.6 words & 0.11 \\
Retail  & 36.6 words & 75.4 words & $< 0.001$ \\
Telecom & 18.0 words & 34.8 words & $< 0.001$ \\
\midrule
\textbf{Pooled} & \textbf{27.2 words} & \textbf{51.7 words} & \textbf{$< 0.001$} \\
\bottomrule
\end{tabular}
\captionof{table}{Comparison of user instruction length in the \texttt{reason\_for\_call} field between simple and complex tasks. Longer descriptions indicate richer user intent and more multi-stage objectives.}
\label{tab:reason_for_call}
\end{minipage}

\end{table*}

To construct a controlled distribution shift for workflow self-evolution, we partition $\tau^2$-Bench tasks into \textit{simple} and \textit{complex} subsets using GPT-5 nano as an LLM classifier. The full prompt is shown in ``Tau2-Bench Task Classification Prompt'' box above. The classifier follows a structured rubric based on workflow complexity, including branching, planning depth, state tracking, verification, and tool usage. Ambiguous cases are assigned a separate \texttt{drop} label instead of being forced into either category. The resulting split statistics are shown in Table~\ref{tab:task_split}.We validate the split using several independent structural signals from $\tau^2$-Bench. As shown in Table~\ref{tab:num_actions}, complex tasks require substantially more mandatory tool-call actions than simple tasks across all domains (5.28 vs.\ 2.47 on average, pooled $p < 0.001$ under a one-sided Mann--Whitney U test). Complex tasks also contain significantly longer user intent descriptions (51.7 vs.\ 27.2 words on average; Table~\ref{tab:reason_for_call}), indicating richer multi-stage objectives. In addition, complex tasks generally involve more natural-language behavioral constraints measured through NL assertions (Table~\ref{tab:nl_assertions}).Finally, across downstream workflow evolution experiments, agents consistently achieve higher success rates on the simple subset than on the complex subset, further confirming that the split captures genuine execution difficulty rather than superficial textual variation. Together, these results support the validity of the LLM-assisted simple/complex partition.

\subsection{Implementation Details}

\begin{table}[t]
\centering
\footnotesize
\setlength{\tabcolsep}{6pt}
\begin{tabular}{lc}
\toprule
\textbf{Hyperparameter} & \textbf{Value} \\
\midrule
Optimizer backbone & GPT-5.1 / GPT-5 nano \\
Executor backbone & GPT-5.1 / GPT-5 nano \\
User simulator & GPT-5.1 / GPT-5 nano \\
Optimizer temperature & 0.2 \\
Executor temperature & 0.2 \\
Maximum workflow rounds & 3 \\
Validation rounds & 1 \\
Evaluation rounds & 1 \\
Maximum tool steps & 40 \\
Max concurrency & 5 \\
Top prior rounds sampled & 3 \\
Complex-task train ratio & 0.2 \\
Complex-task eval ratio & 0.8 \\
\bottomrule
\end{tabular}
\caption{Workflow evolution hyperparameters.}
\label{tab:workflow_impl}
\end{table}

Following the continual self-evolution setting in the main paper, both vanilla and CPE workflows are initialized from the same simple-task seed workflow and subsequently evolved using the first 20\% of the complex-task subset. The remaining 80\% of complex tasks are reserved for held-out evaluation. To measure retained capability stability, we additionally evaluate the evolved workflows on the full simple-task subset. Evaluation uses the $\tau^2$-Bench execution environment with a maximum of 40 interaction steps per episode. During evaluation, the best workflow round is automatically selected according to the highest validation score recorded in \texttt{results.json}. The hyperparameters used in our experiments are summarized in Table~\ref{tab:workflow_impl}; all remaining settings follow the default EvoAgentX configuration.

\subsection{Additional Results}

\begin{figure}[h]
\centering
\begin{subfigure}{0.48\linewidth}
    \centering
    \includegraphics[width=\linewidth]{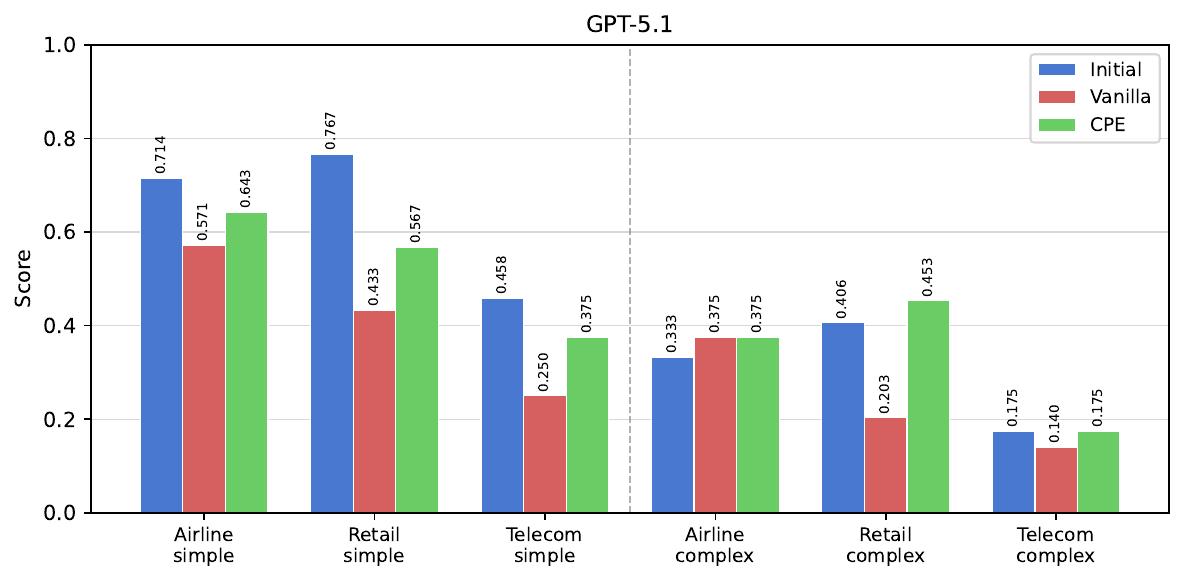}
    \caption{GPT 5.1}
\end{subfigure}
\hfill
\begin{subfigure}{0.48\linewidth}
    \centering
    \includegraphics[width=\linewidth]{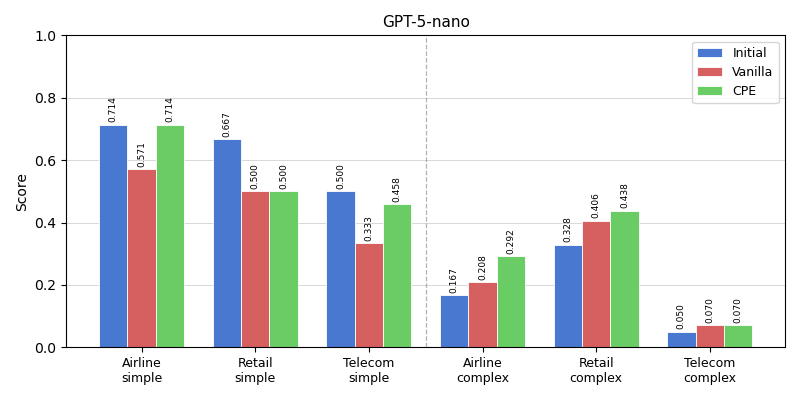}
    \caption{GPT-5 nano}
\end{subfigure}

\caption{Full workflow evolution results across backbones.}
\label{fig:full_workflow_evolve}
\vspace{-8pt}
\end{figure}

Figure~\ref{fig:full_workflow_evolve} presents the full workflow evolution trajectories across GPT-5.1 and GPT-5 nano backbones. Across both models, we observe a consistent pattern: unconstrained workflow self-evolution progressively drifts toward increasingly complex execution structures, which improves adaptation on harder tasks but degrades previously retained simple-task capability. In contrast, CPE substantially stabilizes this evolution process by constraining destructive workflow drift, leading to consistently stronger retention of simple-task performance while maintaining competitive or improved complex-task adaptation. These results further support our central claim that capability erosion emerges as a general property of long-horizon workflow self-evolution, and that capability-preserving regularization improves the stability--plasticity trade-off across different optimizer backbones.

\section{Skill/Tool Evolve Implementation Details}

\subsection{Implementation Details}
\label{app:tool_impl}
We study skill/tool self-evolution using a MemSkill~\citep{zhang2026memskilllearningevolvingmemory} style framework on the MATH benchmark~\citep{hendrycks2021measuringmathematicalproblemsolving}. The continual evolution trajectory follows three sequential domains:
\textit{Algebra} $\rightarrow$ \textit{Geometry} $\rightarrow$ \textit{Number Theory}. 
Each stage uses 500 training samples, and evaluation is performed on up to 1,000 held-out problems per domain. No model parameter updates are performed during evolution; adaptation occurs entirely through skill repository updates. 

The hyperparameters used in our experiments are summarized in Table~\ref{tab:tool_impl}; all remaining settings follow the default configuration. The skill repository is implemented as a bounded-capacity skill bank with a maximum size of 30 skills. Each skill contains a textual description, instruction template, update type, and usage statistics. Skill retrieval uses Contriever embeddings with cosine similarity, while the executor retrieves the top-5 relevant skills and top-5 memories during training. During evaluation, the memory retrieval budget is increased to top-10. The designer module is triggered every 50 chunks and can add at most 3 new skills per update cycle with 2 reflection iterations. 

\begin{table}[t]
\centering
\footnotesize
\setlength{\tabcolsep}{5pt}
\begin{tabular}{lc}
\toprule
\textbf{Hyperparameter} & \textbf{Value} \\
\midrule
Backbone LLM            & GPT-5 nano/GPT-4o mini \\
Skill bank capacity     & 30 \\
Retriever backbone      & Contriever \\
Skill retrieval top-$k$ & 5 \\
Memory retrieval top-$k$ & 5 \\
Eval memory top-$k$     & 10 \\
Chunk eval frequency    & 50 \\
Max new skills / cycle  & 3 \\
Reflection cycles       & 2 \\
Failure pool size       & 200 \\
QA workers              & 20 \\
\bottomrule
\end{tabular}
\caption{Implementation hyperparameters for skill/tool self-evolution experiments.}
\label{tab:tool_impl}
\end{table}

We evaluate both vanilla self-evolution and CPE under the same continual learning protocol. After training on each domain, the resulting skill repository is carried forward to the next domain without replay or parameter reset. Retained capability is evaluated by re-testing earlier domains after later-stage evolution. Final performance is reported after the full Algebra $\rightarrow$ Geometry $\rightarrow$ Number Theory evolution trajectory. Forgetting is measured as the degradation on earlier domains after adaptation to later domains. 

\subsection{Additional Results}
\begin{figure}[h]
\centering
\begin{subfigure}{0.48\linewidth}
    \centering
    \includegraphics[width=\linewidth]{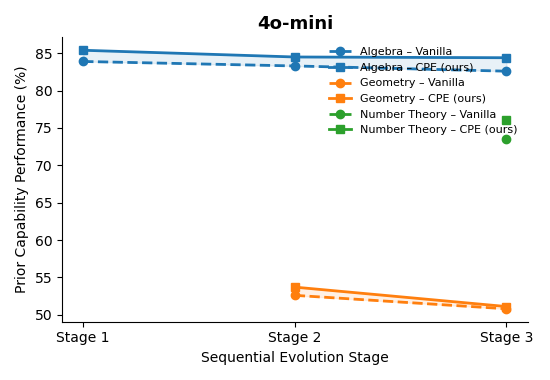}
    \caption{GPT-4o mini}
\end{subfigure}
\hfill
\begin{subfigure}{0.48\linewidth}
    \centering
    \includegraphics[width=\linewidth]{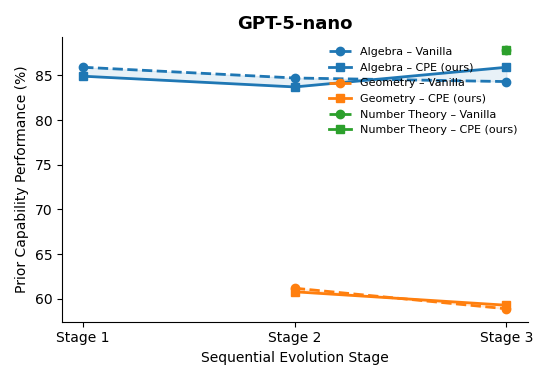}
    \caption{GPT-5 nano}
\end{subfigure}

\caption{Full skill/tool evolution results across backbones.}
\label{fig:full_tool_evolve}
\vspace{-8pt}
\end{figure}

Figure~\ref{fig:full_tool_evolve} presents the complete skill/tool evolution results across both GPT-4o mini and GPT-5 nano backbones. Across all evaluated domains, CPE consistently achieves stronger retained capability stability than vanilla self-evolution under the same sequential evolution trajectory and repository budget. The improvements are particularly visible in domains encountered earlier in the evolution process, where unconstrained repository updates progressively suppress previously useful procedural skills. These results further support the observation that capability erosion in skill evolution mainly arises from bounded-capacity repository overwrite, while capability-preserving consolidation helps maintain broader procedural coverage throughout long-horizon self-evolution.

\section{Model Evolution Implementation Details}

\subsection{Implementation Details}
\label{app:model_impl}
\begin{table}[t]
\centering
\small
\setlength{\tabcolsep}{6pt}
\begin{tabular}{lc}
\toprule
\textbf{Configuration} & \textbf{Value} \\
\midrule
Hardware & NVIDIA A100 \\
Base Models & Qwen3-0.6B/Llama3.2-3B \\
Dataset & MedMCQA \\
Quantization & 4-bit NF4 + bfloat16 \\
LoRA Rank ($r$) & 16 \\
LoRA Alpha ($\alpha$) & 32 \\
LoRA Dropout & 0.05 \\
Learning Rate & $2\times10^{-4}$ \\
Epochs per Stage & 1 \\
Per-device Batch Size & 4 \\
Gradient Accumulation & 4 \\
Effective Batch Size & 16 \\
Generation Temperature & 0.7 \\
Top-$p$ & 0.9 \\
Max New Tokens & 192 \\
Training Questions per Stage & 1,000 \\
Evolution Order & Anatomy $\rightarrow$ Biochemistry $\rightarrow$ Dental \\
EWC Regularization Strength ($\lambda$) & 10,000 \\
EWC Decay Factor ($\gamma$) & 0.9 \\
Fisher Estimation Samples & 100 \\
\bottomrule
\end{tabular}
\caption{Model self-evolution configuration used for all experiments.}
\label{tab:model_impl_details}
\end{table}

We implement model self-evolution using a STaR-style self-training pipeline with LoRA fine-tuning on MedMCQA. At each evolution stage, the model generates rationale-answer trajectories on the current medical domain, filters correct trajectories using strict answer matching, and performs supervised fine-tuning on the resulting demonstrations for one epoch. All experiments are conducted on NVIDIA A100 GPUs. 

Both vanilla self-evolution and CPE use identical optimization settings and evolution trajectories. All models are loaded with 4-bit NF4 quantization and bfloat16 computation. During generation, we use temperature $0.7$, top-$p$ $0.9$, and a maximum generation length of 192 tokens. Each stage uses up to 1,000 training questions sampled from the corresponding MedMCQA subject split. The hyperparameters used in our experiments are summarized in Table~\ref{tab:model_impl_details}; all remaining settings follow the default configuration.

For CPE, we apply an Elastic Weight Consolidation (EWC) regularizer on LoRA parameters. After each stage, parameter importance is estimated using a diagonal Fisher approximation computed from the current-stage training data. 

Following standard continual learning evaluation protocols, after each evolution stage we evaluate the updated model on the held-out test splits of all previously encountered domains to measure retained capability under continual adaptation. Final predictions are evaluated using strict exact-match scoring after answer normalization.

\subsection{Additional Results}

\begin{figure}[h]
\centering
\begin{subfigure}{0.48\linewidth}
    \centering
    \includegraphics[width=\linewidth]{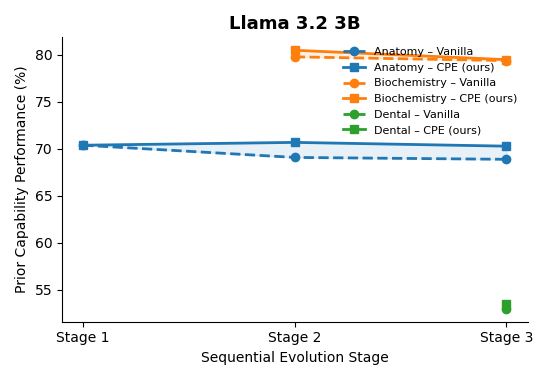}
    \caption{Llama}
\end{subfigure}
\hfill
\begin{subfigure}{0.48\linewidth}
    \centering
    \includegraphics[width=\linewidth]{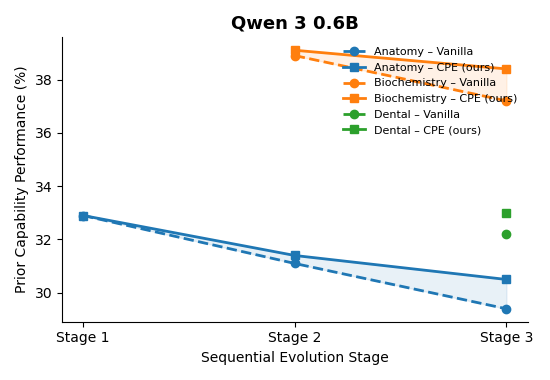}
    \caption{Qwen}
\end{subfigure}

\caption{Full Model evolution results across backbones.}
\label{fig:full_model_evolve}
\vspace{-8pt}
\end{figure}

Figure~\ref{fig:full_model_evolve} presents additional continual self-evolution trajectories for both Llama and Qwen backbones. Consistent with the main results, vanilla self-evolution exhibits progressive degradation on previously learned medical domains as later-stage adaptation proceeds, reflecting continual parametric overwrite under sequential self-training. In contrast, CPE consistently stabilizes retained performance throughout the evolution trajectory while preserving adaptation capability on newly introduced domains. These trends further support that capability erosion is a persistent phenomenon across backbone families rather than a model-specific artifact.

\section{Memory Evolve Implementation Details}
\label{app:memory_impl}

We instantiate memory self-evolution using the Dynamic Cheatsheet framework~\citep{suzgun2025dynamiccheatsheettesttimelearning} under a continual domain adaptation setting. 
The evolving memory repository is implemented as an external adaptive cheatsheet memory that is updated sequentially as the task distribution shifts across mathematical reasoning domains.

We use GPT-4o mini and GPT-5 nano as the underlying reasoning backbones.
For all experiments, the maximum generation length is set to 8192 tokens.
GPT-5 nano is configured with reasoning effort set to \texttt{medium} and text verbosity set to \texttt{low}. 
All experiments use random seed 42. The hyperparameters used in our experiments are summarized in Table~\ref{tab:memory_impl}; all remaining settings follow the default configuration. 

The continual evolution trajectory consists of sequential adaptation across mathematical domains.
The old-domain distributions contain \textit{Counting \& Probability}, \textit{Prealgebra}, and \textit{Number Theory}, while the new-domain distribution corresponds to \textit{Geometry}, which exhibits larger structural divergence from the previous domains.

For each evolution stage, both the old-domain and new-domain memory repositories are constructed using 30 training problems per domain.
The retained old-domain evaluation set contains 200 held-out problems.

Following the continual self-evolution setting described in the main paper, we evaluate memory evolution under both retained-capability and forward-adaptation criteria.
After sequential memory updates on the new-domain stream, we evaluate:
(i) retained performance on previously encountered domains using the preserved memory repository, and 
(ii) adaptation performance on the newly introduced domain.

% To isolate memory evolution effects, all other components including the backbone model, prompting strategy, optimizer configuration, and inference uses default settings to be identical between vanilla self-evolution and CPE-based evolution.
% The evaluation protocol follows the same retained-capability comparison framework used throughout the paper, where old-domain probe accuracy is measured before and after new-domain memory evolution.

\begin{table}[t]
\centering
\small
\setlength{\tabcolsep}{6pt}
\begin{tabular}{ll}
\toprule
\textbf{Component} & \textbf{Configuration} \\
\midrule
Backbone Models & GPT-5 nano/GPT-4o mini \\
Max Tokens & 8192 \\
GPT-5 Reasoning Effort & medium \\
GPT-5 Verbosity & low \\
Training Size & 30 problems per domain \\
Old-Domain Test Size & 200 problems \\
Sequential Domains & Count.\&Prob. $\rightarrow$ Prealg. \\
& $\rightarrow$ Num. Theory $\rightarrow$ Geometry \\
% Random Seed & 42 \\
\bottomrule
\end{tabular}
\caption{
Implementation details for memory self-evolution experiments.
}
\label{tab:memory_impl}
\end{table}

\section{Examples of Evolved Skills in Skill Self-Evolve}
\begin{tcolorbox}[colback=white,colframe=blue!40!gray,title=\textbf{triangulation\_area\_template}]
\textbf{Description:} Compute the area of a polygon by dividing it into triangles and summing their areas.

\vspace{0.5em}
\textbf{Instruction Template:}

Divide the polygon into triangles.  
Compute the area of each triangle using the given base and height.  
Sum the triangle areas to obtain the total polygon area.
\end{tcolorbox}

\begin{tcolorbox}[colback=white,colframe=blue!40!gray,title=\textbf{triangulation\_area\_template\_v3}]
\textbf{Description:} Robust triangulation workflow for complex polygons: perform a fan or ear-clipping triangulation from a fixed vertex, compute each triangle area, sum all sub-areas, optionally verify with the shoelace formula, and cross-check with total area constraints.

\vspace{0.5em}
\textbf{Instruction Template:}

(1) Choose a fixed vertex and triangulate the polygon into non-overlapping triangles.  
(2) For each triangle, compute its area via base-height or cross product.  
(3) Sum all triangle areas.  
(4) If coordinates are available, verify using the shoelace formula.  
(5) Check that the total is consistent with any given area constraints.  
(6) Return the final area.
\end{tcolorbox}

\begin{tcolorbox}[colback=white,colframe=blue!40!gray,title=\textbf{area\_union\_and\_overlap\_template\_for\_complex\_regions\_v2}]
\textbf{Description:} A robust workflow for computing areas defined as unions or intersections of multiple overlapping shapes (circles, sectors, polygons), with explicit partitioning and inclusion-exclusion reasoning.

\vspace{0.5em}
\textbf{Instruction Template:}

When a problem defines a region as a union or intersection of multiple shapes:

(a) Partition the plane into disjoint cells where region membership is constant, compute each cell area, and sum them.  
(b) Apply inclusion-exclusion to correct for overlaps where necessary.  
(c) Use symmetry to reduce repetitive work and avoid double counting.  
(d) Cross-check with an alternate decomposition such as shoelace or triangulation when coordinates are available.  
(e) For circle-based regions, explicitly decompose into sectors and segments using central angles.  
(f) Provide the final consolidated area, and if needed, surface a canonical fallback workflow with explicit validation constraints.
\end{tcolorbox}

\begin{tcolorbox}[colback=white,colframe=blue!40!gray,title=\textbf{problem\_type\_tagging\_and\_ranking}]
\textbf{Description:} Tag memories with problem-type categories and rank their relevance to the current problem in order to retrieve the most helpful strategies.

\vspace{0.5em}
\textbf{Instruction Template:}

(1) Annotate each memory with a set of problem-type tags such as quadratic, geometry, polynomial, inverse\_function, or geometric\_progression.  
(2) For the current problem, compute a relevance score using tag overlap, keyword overlap, and historical success weight.  
(3) Retrieve the top-$k$ memories with the highest relevance and present them with a brief one-line relevance justification.  
(4) If a memory is broadly useful but procedurally vague, include an applicability note clarifying the problem types it best addresses.  
(5) Prefer memories that provide concrete solving steps or a short actionable plan aligned with the current problem type.
\end{tcolorbox}

\section{Use of AI Assistants}
\label{app:llm}

AI assistants were used as auxiliary tools for manuscript preparation, including language polishing, clarity improvement, organization, and limited experimental workflows. 
All experimental design, methodological decisions, analyses, reported results, and final content were reviewed and verified by the authors.

\section{Limitations}
\label{app:limitations}
Although we propose Capability-Preserving Evolution (CPE) as a general principle for stabilizing self-evolving agents, the current work only explores relatively straightforward constraint-based instantiations within each evolution dimension. Our goal is primarily to verify the existence of capability erosion under self-evolution and demonstrate that explicitly regularizing destructive capability drift can substantially improve retained capability stability.

Accordingly, the preservation mechanisms used in this paper are intentionally lightweight and domain-specific, including anchor-based workflow regularization, repository consolidation, Fisher-style parameter constraints, and memory preservation heuristics. More sophisticated preservation strategies, adaptive regularization schedules, or jointly optimized multi-component preservation frameworks may further improve the stability--plasticity tradeoff in long-horizon self-evolving agents.

In addition, our experiments focus on controlled sequential task-distribution shifts and representative self-evolution frameworks. Real-world autonomous agents may encounter substantially more complex and open-ended evolution trajectories, including dynamic environments, multi-agent interaction, and continuously expanding capability spaces. Extending CPE to such settings remains an important direction for future work.

\section{Broader Impacts}
\label{app:broader_impacts}
This work studies capability degradation in self-evolving LLM agents and proposes Capability-Preserving Evolution (CPE) as a stabilization principle for continual adaptation. A potential positive impact of this work is improving the reliability and robustness of long-horizon autonomous agents by reducing destructive capability drift during self-evolution. More stable self-evolving systems may be beneficial for applications that require continual adaptation while maintaining previously acquired behaviors.

At the same time, improving the stability of self-evolving agents may also increase the robustness and persistence of autonomous systems in settings where misuse, over-automation, or unintended behaviors could create societal risks. In particular, more reliable autonomous adaptation mechanisms could be incorporated into large-scale agentic systems with limited human oversight. Our work does not introduce new deployment pipelines or application-specific autonomous capabilities, but the broader risks associated with increasingly capable self-evolving agents should still be considered.

Overall, we view capability preservation as an important component of safer continual adaptation, since uncontrolled capability drift can itself produce unreliable or unstable agent behavior during long-horizon self-evolution.

\end{document}